\documentclass{modelica}
\hypersetup{%
  pdftitle  = {A 2-DOF Multibody Track Joint for Race Car Optimization},
  pdfauthor = {TrackJoint authors},
  pdfsubject = {North American Modelica and FMI Conference 2026},
  pdfkeywords = {Modelica, MultiBody, FMI, race car, track model, Monge patch, kinematic joint, d'Alembert, optimal control},
  colorlinks,
  linkcolor=black,
  urlcolor=black,
  citecolor=black,
  pdfpagelayout = SinglePage,
  pdfcreator = pdflatex,
  pdfproducer = pdflatex}

\usepackage[utf8]{inputenc}
\usepackage{amsmath,amssymb}
\usepackage{amsthm}
\usepackage{bm}

\newtheorem{proposition}{Proposition}

\begin{document}
\thispagestyle{empty}

\title{A 2-DOF MultiBody Track Joint for Race Car Optimization}
\author[1]{Srikanth Sivaramakrishnan}
\author[1]{Behzad Danaei}
\author[1]{Adwait Verulkar}
\author[1]{Ryan Trinter}
\affil[1]{General Motors, USA, {\small\texttt{\{srikanth.sivaramakrishnan, behzad.danaei, adwait.verulkar, ryan.trinter\}@gm.com}}}

\maketitle\thispagestyle{empty}

\abstract{%
\texttt{TrackJoint} is a Modelica MultiBody component that turns a
three-dimensional road surface into a joint with two degrees of freedom.
The two joint coordinates are the distance travelled along a reference
line and the sideways offset from that line. Together they give the
position, orientation and angular velocity of a single \texttt{frame\_b}
connector. All road geometry reaches the joint as a small set of periodic
piecewise-quintic splines, so one component can represent any circuit whose
splines have been fitted beforehand. The road constraint is closed with
d'Alembert's principle. The two generalized forces that
belong to the joint coordinates are set to zero, so no Lagrange multipliers are needed.
The joint is connected to a multibody race car in Dymola and the pair is
exported as one FMI~3.0 FMU. The same FMU is then used for time-domain
simulation and for CasADi/IPOPT minimum lap-time optimization, with no
change to the model in between. This works because the joint keeps the
arc-length coordinate as a real state and publishes its time derivative,
which is the quantity that lets the problem be rewritten with distance
We report results for minimum lap time simulations on
the Singapore Formula~1 circuit.
}


\section{Introduction}

Race-car simulation in Modelica relies on a road model
that generates accurate surface geometry in terms of position,
orientation and angular velocity for each tire contact patch.
Inside a generic MultiBody vehicle simulation this is naturally expressed by
a Modelica \emph{joint} that constrains the chassis of the vehicle
to lie on the road surface while leaving the in-plane road
coordinates free. For ideal flat surfaces this role is filled by
\texttt{Modelica.\allowbreak Mechanics.\allowbreak MultiBody.\allowbreak Joints.\allowbreak Planar}. Although analytically defined ribbons have been used in earlier
race-car work \cite{PerantoniLimebeer:2015,Lovato:2021}, they are
typically hand-coded inside a surrogate optimization model rather
than packaged as a reusable MultiBody component. The idea of a
reusable path-following joint that constrains motion to a
spline-parametrized curve has been formulated symbolically in a
graph-theoretic multibody setting by \textcite{Hall:2014}. 
This method is used to reduce complex automotive suspension linkages into a single degree-of-freedom joint for computational efficiency. 
A more closely related precedent is the
\texttt{Claytex.\allowbreak Mechanics.\allowbreak MultiBody.\allowbreak Joints.\allowbreak Rail}
joint of the commercial Claytex library \cite{Briant:2021}, which constrains a
\texttt{frame\_b} to follow a one-dimensional B-spline rail
parametrized by arc length $s$ and resolves the wrench through
the spline-derivative projection that we also adopt below. The
Claytex component is, however, restricted to a curve: it leaves a
single translational degree of freedom and exposes the orientation
either as a heading frame slaved to the rail tangent or as three
independent angle splines.

This paper presents \texttt{TrackJoint}, a Modelica MultiBody joint
that takes the Claytex \texttt{Rail} idea as its starting point
and extends it from a 1-DOF rail to a 2-DOF Monge-patch racetrack
surface. The body is now free to move both along a planar
reference line ($s$) and laterally across it ($w$); the
orientation triad is constructed analytically from the surface
partials rather than from a separate angle table; and the wrench
is closed by two d'Alembert relations rather than one. The
underlying surface model is the Monge patch
\cite{LimebeerVanDeMerwe:2025,doCarmo:1976}
\begin{equation}
\bm r_b(s,w)
 = \bm r(s) + w\,\bm n_{\!g}(s) + z(s,w)\,\bm e_z,
\label{eq:patch}
\end{equation}
where $\bm r_b(s,w)$ is the world-frame position of the track point
identified by $(s,w)$, $\bm r(s)=(r_x(s),r_y(s),0)$ is the planar reference line,
$\bm n_{\!g}(s)=(-\partial_s r_y(s),\partial_s r_x(s),0)$ its ground-plane left normal
and
\begin{equation}
z(s,w) = r_z(s) + z_1(s)\,w + z_2(s)\,w^2 + z_3(s)\,w^3
\label{eq:zexp}
\end{equation}
the lateral cubic in $w$ proposed by
\textcite{LimebeerVanDeMerwe:2025} for elevation. The six channel
splines $\{r_x,r_y,r_z,z_1,z_2,z_3\}$ are obtained from a separate
fitting pipeline that processes a measured racetrack point cloud and
exports them as periodic piecewise-quintic ($C^4$ at the knots)
B-splines anchored on the minimum-curvature line of the circuit
\cite{Verulkar:2026fit,Heilmeier:2020,Braghin:2008}. Here
$(\bm e_x,\bm e_y,\bm e_z)$ denotes the fixed world Cartesian basis.
The present paper takes those splines as given and concentrates on the Modelica
joint that consumes them.

The contribution of this paper is a Modelica implementation of a
track joint that exposes the road surface as a single reusable
MultiBody component suitable for both time-domain simulation and
optimization. For efficiency the per-step evaluation is
implemented with Horner expansion of the Monge cubic, constant-time
segment lookup on the uniform break vector and a shared spline
evaluator, so that the same component is cheap to evaluate at every collocation point of a CasADi/IPOPT
\cite{Andersson:2019,WaechterBiegler:2006} minimum lap-time
problem. The model also ensures sufficient accuracy to the LiDAR measurements of the track and can be used in time-domain simulation of a
multibody race car driven around a closed circuit.

The remainder of the paper is organized as follows.
Section~\ref{sec:surface} recalls the Monge-patch surface
formulation and fixes the notation used by the joint.
Section~\ref{sec:joint} derives the joint equations from
\eqref{eq:patch}--\eqref{eq:zexp}.
Section~\ref{sec:implementation} discusses the Modelica
implementation and the FMI export. Section~\ref{sec:results}
reports the two downstream use cases.

\section{Surface model}\label{sec:surface}

The road is modeled as a smooth orientable two-dimensional manifold
$\mathcal M\subset\mathbb R^3$ parametrized by~\eqref{eq:patch}. We
reproduce here only the relations needed by the Modelica
implementation. A full geometric discussion is given in
\cite{LimebeerVanDeMerwe:2025} and the companion fitting paper
\cite{Verulkar:2026fit}.

Differentiating the surface position once gives the in-surface
tangent vectors
\begin{align}
\partial_s\bm r_b
 &= \bigl(\partial_s r_x-w\,\partial_{ss} r_y\bigr)\bm e_x
   + \bigl(\partial_s r_y+w\,\partial_{ss} r_x\bigr)\bm e_y
   + \partial_s z\,\bm e_z,
\label{eq:xPs}\\[2pt]
\partial_w\bm r_b
 &= -\partial_s r_y\,\bm e_x + \partial_s r_x\,\bm e_y + \partial_w z\,\bm e_z,
\label{eq:xPw}
\end{align}
in which the lateral derivatives $\partial_s z$ and $\partial_w z$ follow from
the cubic~\eqref{eq:zexp},
\begin{align}
\partial_s z &= \partial_s r_z + \partial_s z_1\,w + \partial_s z_2\,w^2 + \partial_s z_3\,w^3,\\
\partial_w z &= z_1 + 2\,z_2\,w + 3\,z_3\,w^2.
\end{align}
The surface unit normal is
\begin{equation}
\bm n_z
 = \frac{\partial_s\bm r_b\times\partial_w\bm r_b}
   {\Vert\partial_s\bm r_b\times\partial_w\bm r_b\Vert},
\label{eq:enP}
\end{equation}
and a right-handed surface triad is
$(\bm n_x,\bm n_y,\bm n_z)
 =\bigl(\partial_s\bm r_b/\Vert\partial_s\bm r_b\Vert,\;
   \bm n_z\times\partial_s\bm r_b/\Vert\partial_s\bm r_b\Vert,\;
        \bm n_z\bigr)$,
identical to that adopted in \cite{LimebeerVanDeMerwe:2025}.

The constraint that keeps the attached body on $\mathcal M$ leaves two
degrees of freedom, namely the parameters $s$ and $w$. The Modelica joint
carries $(s,w)$ and their first time derivatives
$(\dot s,\dot w)$ as states, and works everything else out from the
splines $\{r_x,r_y,r_z,z_1,z_2,z_3\}$ at the current $s$.

\subsection{Regularity of the parametrization}\label{sec:regularity}

The joint only works where the patch \eqref{eq:patch} is locally
invertible. That puts a condition on the reference line. Let
$\kappa(s)$ be the signed planar curvature of $\bm r(s)$. The cross product
of the in-surface tangents \eqref{eq:xPs}--\eqref{eq:xPw} then has magnitude
\begin{equation}
\Vert\partial_s\bm r_b\times\partial_w\bm r_b\Vert
 \;=\;\bigl|1-\kappa(s)\,w\bigr|\,\sqrt{1+\Vert\nabla z\Vert^2}
 \;+\;\mathcal O(\partial_s z),
\label{eq:jac}
\end{equation}
so the patch breaks down on the fold
\begin{equation}
w_{\mathrm{fold}}(s)=\frac{1}{\kappa(s)},
\label{eq:fold}
\end{equation}
which is the centre of curvature of the reference line. Suppose the
reference line is the geometric centre-line. In any corner whose radius of
curvature drops below the local half-width $b(s)$ of the road we get
$|w_{\mathrm{fold}}(s)|<b(s)$, so the fold sits \emph{inside} the drivable
strip. There the map $(s,w)\mapsto\bm r_b$ stops being one-to-one, and the
joint hands back two different points for the same $(s,w)$ along with a
Jacobian that goes to zero. This happens in the slow corners, which are
the ones that decide lap time. Anchoring the splines on the
minimum-curvature line lowers the curvature peaks and moves the fold
outside the strip. The joint therefore needs its input data to satisfy
\begin{equation}
\min_{s}\ \frac{1}{|\kappa(s)|\,b(s)} \;>\; 1 ,
\label{eq:injectivity}
\end{equation}
which we check for every circuit used in this paper in
Section~\ref{sec:benchmarks}.

\subsection{Continuity required of the splines}\label{sec:continuity}

The joint needs high derivatives of its inputs. The rotation rates in
Section~\ref{sec:rates} use the third $s$-derivative of the planar
channels $r_x$ and $r_y$. The Modelica tool then differentiates the
angular-velocity equation once more while it reduces the DAE index, which
takes the requirement up to a fourth derivative. A piecewise quintic with
$C^4$ continuity at the interior knots and at the wrap is the lowest degree
that keeps all of these continuous across a knot. A cubic spline leaves
$\partial_{ss}\bm r_b$ jumping at every knot, which makes the integrator
chatter at the segment boundaries. That is why the fitting pipeline
\cite{Verulkar:2026fit} exports quintics even though cubics would be enough
for the surface on its own.

\section{Joint formulation}\label{sec:joint}

\begin{figure}[t]
  \centering
  \includegraphics[width=0.48\textwidth]{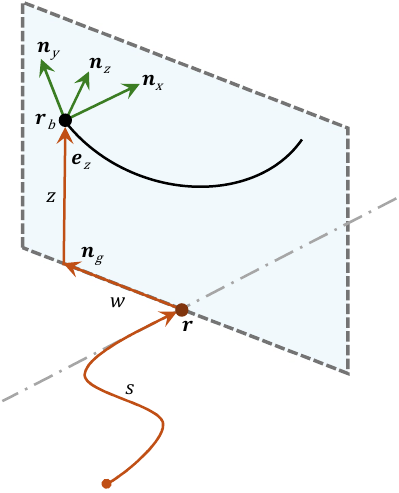}
  \caption{Schematic of the \texttt{TrackJoint}. The body
  \texttt{frame\_b} is constrained to the Monge-patch surface
  $\mathcal M$ and is free to move along the reference line
  ($s$) and laterally across it ($w$). At the current $(s,w)$ the
  joint evaluates the six channel splines
  $\{r_x,r_y,r_z,z_1,z_2,z_3\}$ to build the surface position
  $\bm r_b(s,w)$ and the right-handed surface triad
  $(\bm n_x,\bm n_y,\bm n_z)$, whose tangent $\bm n_x$ follows the
  reference line and whose normal $\bm n_z$ is the surface
  normal. The reaction wrench is resolved onto the two free
  directions through the d'Alembert projection.}
  \label{fig:trackjoint}
\end{figure}

Figure~\ref{fig:trackjoint} illustrates the overall working
principle of the joint. The body \texttt{frame\_b} is pinned to the
road surface $\mathcal M$ and retains exactly the two in-plane
degrees of freedom $(s,w)$. From these coordinates the joint
reconstructs the full pose and angular velocity of the connector and
projects the constraint forces onto the surface, as detailed in the
subsections below.

\subsection{States and frame outputs}

The joint has two free coordinates $\bm q=(s,w)$ with derivatives
$\bm{\dot q}=(\dot s,\dot w)$. Their values at any instant determine
the position, orientation and angular velocity of the
\texttt{frame\_b} connector. The position is given directly by
\eqref{eq:patch} as $\bm r_b(s,w)$, while orientation and
angular velocity follow
\begin{align}
\mathbf R_b &= [\,\bm n_x\;\;\bm n_y\;\;\bm n_z\,]^{\!\top},
\label{eq:Rb}\\
\bm\omega_b
 &= \mathbf R_b\bigl(\partial_s\bm\phi\,\dot s
                    +\partial_w\bm\phi\,\dot w\bigr),
\label{eq:omegab}
\end{align}
where $\partial_s\bm\phi,\partial_w\bm\phi\in\mathbb R^3$ are the parametric
rotation-rate vectors of the surface triad with respect to $s$ and
$w$, and $\mathbf R_b$ is the rotation matrix associated with
that frame; and $\bm\omega_b$ is its world-frame angular velocity.
Equation~\eqref{eq:omegab} expresses the world-frame angular
velocity as a linear combination of two parametric angular
velocities, exactly as in the standard MultiBody treatment of any
two-parameter rigid kinematic chain.

\subsection{Parametric rotation rates}\label{sec:rates}

The rotation-rate vectors are obtained in closed form from the
surface partials. Let $\bm\xi\in\{s,w\}$ denote either parameter.
Since the triad $\mathbf R_b=[\bm n_x\;\bm n_y\;\bm n_z]^{\!\top}$ is
orthonormal for every $(s,w)$, its parametric derivative is
skew-symmetric, so there exists a unique vector
$\partial_\xi\bm\phi$ with
\begin{equation}
\partial_\xi\bm n_i=\partial_\xi\bm\phi\times\bm n_i,
\qquad i\in\{x,y,z\},
\label{eq:rotdef}
\end{equation}
and $\partial_\xi\bm\phi$ may be recovered from any two columns by
$\partial_\xi\bm\phi=\tfrac12\sum_i\bm n_i\times\partial_\xi\bm n_i$.
Building it from the two pieces that the surface partials already give is
cheaper and better conditioned, and is implemented in this code.

Write $\bm c=\partial_s\bm r_b\times\partial_w\bm r_b$ for the
un-normalized binormal, so that $\bm n_z=\bm c/\Vert\bm c\Vert$.
Differentiating this quotient gives
\begin{equation}
\partial_\xi\bm n_z
 =\frac{\partial_\xi\bm c}{\Vert\bm c\Vert}
  -\frac{(\bm n_z\!\cdot\!\partial_\xi\bm c)}{\Vert\bm c\Vert}\,\bm n_z
 =\frac{1}{\Vert\bm c\Vert}\,
   \bigl(\mathbf I-\bm n_z\bm n_z^{\!\top}\bigr)\partial_\xi\bm c ,
\label{eq:dnz}
\end{equation}
the component of $\partial_\xi\bm c$ tangent to the surface. Compare
\eqref{eq:dnz} with \eqref{eq:rotdef} for $i=z$ and use
$\bm n_z\times(\bm n_z\times\bm v)=-(\mathbf I-\bm n_z\bm n_z^{\!\top})\bm v$.
This identifies the two in-plane components of $\partial_\xi\bm\phi$ as
$\bm n_z\times\partial_\xi\bm c/\Vert\bm c\Vert$. That term turns the triad
about axes lying in the surface, so it is the roll and pitch part. It says
nothing about the component along $\bm n_z$, because turning about
$\bm n_z$ does not move $\bm n_z$.

That remaining component is fixed by the tangent column. From
$\bm n_x=\partial_s\bm r_b/\Vert\partial_s\bm r_b\Vert$ and
\eqref{eq:rotdef} for $i=x$,
\begin{equation}
\partial_\xi\bm\phi\!\cdot\!\bm n_z
 =\bigl(\bm n_x\times\partial_\xi\bm n_x\bigr)\!\cdot\!\bm n_z
 =\frac{(\partial_\xi\partial_s\bm r_b)\!\cdot\!\bm n_y}
        {\Vert\partial_s\bm r_b\Vert},
\label{eq:yawpart}
\end{equation}
where the normalization factor drops out because only the part of
$\partial_\xi\partial_s\bm r_b$ orthogonal to $\bm n_x$ survives the
projection onto $\bm n_y$. This is the turn about the surface normal, so it
is the yaw part. Adding the two contributions and setting $\xi$ to $s$ and
then to $w$ gives
\begin{align}
\partial_s\bm\phi
 &= \frac{\bm n_z\times\partial_s\bm c}{\Vert\bm c\Vert}
     + \frac{(\partial_{ss}\bm r_b)\!\cdot\!\bm n_y}
       {\Vert\partial_s\bm r_b\Vert}\,\bm n_z,
\label{eq:phis}\\[2pt]
\partial_w\bm\phi
 &= \frac{\bm n_z\times\partial_w\bm c}{\Vert\bm c\Vert}
     + \frac{(\partial_{sw}\bm r_b)\!\cdot\!\bm n_y}
       {\Vert\partial_s\bm r_b\Vert}\,\bm n_z.
\label{eq:phiw}
\end{align}
The required derivatives of $\bm c$ follow from the product rule,
\begin{align}
\partial_s\bm c
 &= (\partial_{ss}\bm r_b)\times\partial_w\bm r_b
   + \partial_s\bm r_b\times(\partial_{sw}\bm r_b),\\
\partial_w\bm c
 &= (\partial_{sw}\bm r_b)\times\partial_w\bm r_b
   + \partial_s\bm r_b\times(\partial_{ww}\bm r_b),
\end{align}
which combined with~\eqref{eq:xPs}--\eqref{eq:xPw} only require the
spline values up to the third $s$-derivative ($\partial_{sss} r_x$, $\partial_{sss} r_y$),
the second $s$-derivative of the height channels ($\partial_{ss} r_z,\partial_{ss} z_k$) and
the lateral derivatives of $z(s,w)$ up to second order
($k\in\{1,2,3\}$ for $z_k$).

\subsection{Time-to-space transformation for minimum lap-time problems}\label{sec:timespace}

A minimum lap-time problem is naturally posed with the arc-length
coordinate $s$ rather than physical time $t$ as the independent
variable, so that the integration domain $s\in[0,L]$ is fixed and
the lap time $T=\int_0^L\mathrm dt$ becomes the cost ($L$ is the
track length). The change
of variable rests on the single scalar relation
\begin{equation}
\frac{\mathrm dt}{\mathrm ds} = \frac{1}{\dot s},
\qquad
\dot s = \frac{v_t\cos\chi - v_n\sin\chi}{1-\kappa(s)\,w},
\label{eq:dtds}
\end{equation}
where $(v_t,v_n)$ are the chassis velocity components in the track
triad, $\chi$ the chassis yaw relative to the reference-line
tangent, and $\kappa(s)$ the planar curvature of the reference line.
Every other state derivative $\dot{\bm x}$ in the
vehicle model is then rescaled as
$\mathrm d\bm x/\mathrm ds = \dot{\bm x}/\dot s$ before being
handed to the collocation transcription. The factor $\dot s$ is
therefore the single quantity through which the road geometry
enters the time-to-space transformation.

The present joint exports $\dot s$ as the state variable
\texttt{vsTrack}~$=$~\texttt{der(sTrack)}, declared with
\texttt{StateSelect.always} so that any Modelica tool retains it
as a continuous output of the FMU into which the host vehicle
model and the joint are exported together. The FMIOPT
transcription reads this variable directly from the exported
FMI~3.0 interface and uses it to rescale every other state
derivative without any further model modification. It supplies the one quantity
the time-to-space transformation needs, which makes it a prerequisite for
the minimum lap-time application reported in
Section~\ref{sec:results}.

\subsection{D'Alembert constraint}

Rather than introducing two Lagrange multipliers for the orthogonal
directions to $\mathcal M$, we close the kinematic chain by setting
the two generalized forces conjugate to $(s,w)$ to zero. Resolving
the body-frame wrench $(\bm f_b,\bm\tau_b)$ on the
\texttt{frame\_b} connector to the world orientation,
$\bm f^w_b=\mathbf R_b^{\!\top}\bm f_b,
 \bm\tau^w_b=\mathbf R_b^{\!\top}\bm\tau_b$,
the principle of virtual work gives
\begin{align}
0 &= \bm f^w_b\!\cdot\!\partial_s\bm r_b
    +\bm\tau^w_b\!\cdot\!\partial_s\bm\phi,
\label{eq:dAlembertS}\\
0 &= \bm f^w_b\!\cdot\!\partial_w\bm r_b
    +\bm\tau^w_b\!\cdot\!\partial_w\bm\phi.
\label{eq:dAlembertW}
\end{align}
Equations~\eqref{eq:dAlembertS}--\eqref{eq:dAlembertW} provide the
two algebraic equations needed to close the index-1 DAE in
$(\dot s,\dot w)$ once the vehicle dynamics couple the wrench to the
states. They are linear in the six wrench components
$(\bm f^w_b,\bm\tau^w_b)$ and can be written as
$\mathbf A\,[\bm f^w_b;\bm\tau^w_b]=\bm 0$ with the $2\times6$ matrix
\begin{equation}
\mathbf A=
\begin{bmatrix}
(\partial_s\bm r_b)^{\!\top} & (\partial_s\bm\phi)^{\!\top}\\[2pt]
(\partial_w\bm r_b)^{\!\top} & (\partial_w\bm\phi)^{\!\top}
\end{bmatrix}.
\label{eq:Amat}
\end{equation}
The DAE is solvable when $\operatorname{rank}\mathbf A=2$, and that depends
only on the surface geometry.

\begin{proposition}\label{prop:rank}
$\operatorname{rank}\mathbf A=2$ if and only if
$\partial_s\bm r_b$ and $\partial_w\bm r_b$ are linearly independent,
i.e.\ if and only if
$\Vert\partial_s\bm r_b\times\partial_w\bm r_b\Vert\neq0$.
\end{proposition}

\begin{proof}
The first three columns of $\mathbf A$ are
$[\partial_s\bm r_b,\partial_w\bm r_b]^{\!\top}$. If these two vectors
are independent that block already has rank two, so
$\operatorname{rank}\mathbf A=2$. Conversely, suppose
$\partial_w\bm r_b=\lambda\,\partial_s\bm r_b$ for some
$\lambda\in\mathbb R$. Then $\bm c=\bm 0$, the triad is undefined, and
by \eqref{eq:phis}--\eqref{eq:phiw} the rate vectors degenerate to
$\partial_w\bm\phi=\lambda\,\partial_s\bm\phi$ as well, so the second
row of $\mathbf A$ is $\lambda$ times the first and
$\operatorname{rank}\mathbf A\le1$.
\end{proof}

By \eqref{eq:jac} the condition in Proposition~\ref{prop:rank} fails
exactly on the fold \eqref{eq:fold}. So the Frenet fold reported in
\cite{Verulkar:2026fit} is more than an awkward piece of geometry. It is
the configuration in which the joint's constraint Jacobian loses rank and
the DAE can no longer be solved. Condition \eqref{eq:injectivity} on the
input splines is what keeps $\mathbf A$ at full rank across the whole
drivable strip.

\subsection{Compact form of the implementation}

Listing~\ref{lst:trackjoint} reproduces the equation block of the
Modelica component. It evaluates the spline channels at the current
$s$, expands the Monge cubic in Horner form, assembles the surface
partials~\eqref{eq:xPs}--\eqref{eq:xPw} and their parametric
derivatives, builds the triad and the parametric angular velocity
vectors~\eqref{eq:phis}--\eqref{eq:phiw}, and asserts the
constraints~\eqref{eq:dAlembertS}--\eqref{eq:dAlembertW}.

\begin{lstlisting}[caption={Equation block of \texttt{TrackJoint}.
The reference line and height channels are evaluated by a piecewise
spline function bound to the current segment index \texttt{ix}. The
Monge cubic and its partial derivatives are computed in Horner
form. The surface triad and the parametric rotation-rate vectors
follow directly, and the d'Alembert relations close the 2-DOF DAE
without Lagrange multipliers.},label=lst:trackjoint]

  // Height z(s,w) in Horner form and partial derivatives
  z = rz + wTrack*(z1 + wTrack*(z2 + wTrack*z3));
  dzds = drzds + wTrack*(dz1ds + wTrack*(dz2ds + wTrack*dz3ds));
  dzdw = z1 + wTrack*(2*z2 + wTrack*(3*z3));
  d2zds2 = d2rzds2 + wTrack*(d2z1ds2 + wTrack*(d2z2ds2 + wTrack*d2z3ds2));
  d2zdsdw = dz1ds + wTrack*(2*dz2ds + wTrack*(3*dz3ds));
  d2zdw2 = 2*z2 + wTrack*(6*z3);

  // Surface position and partial derivatives
  rTrack = {rx - wTrack*dryds, ry + wTrack*drxds, z};
  drds = {drxds - wTrack*d2ryds2, dryds + wTrack*d2rxds2, dzds};
  drdw = {-dryds, drxds, dzdw};
  d2rds2 = {d2rxds2 - wTrack*d3ryds3, d2ryds2 + wTrack*d3rxds3, d2zds2};
  d2rdsdw = {-d2ryds2, d2rxds2, d2zdsdw};
  d2rdw2  = {0, 0, d2zdw2};

  // Track triad: tangent, normal, binormal
  nxTrack = normalize(drds);
  nzTrack = normalize(cross(drds, drdw));
  nyTrack = cross(nzTrack, nxTrack);

  // Parametric angular-rate vectors (compact triad transport)
  cTrack = cross(drds, drdw);
  dcTrackds = cross(d2rds2,  drdw) + cross(drds, d2rdsdw);
  dcTrackdw = cross(d2rdsdw, drdw) + cross(drds, d2rdw2);
  dphids = cross(nzTrack, dcTrackds)/length(cTrack) + ((d2rds2 *nyTrack)/length(drds))*nzTrack;
  dphidw = cross(nzTrack, dcTrackdw)/length(cTrack) + ((d2rdsdw*nyTrack)/length(drds))*nzTrack;

  // Frame_b angular velocity from parametric rates
  derPhiTrack = dphids*vsTrack + dphidw*vwTrack;

  // D'Alembert: virtual work in s and w directions
  0 = fFrameB*drds + tFrameB*dphids;
  0 = fFrameB*drdw + tFrameB*dphidw;

  // Frame_b: position, rotation matrix, angular velocity
  frame_b.r_0 = rTrack;
  frame_b.R.T = transpose([nxTrack, nyTrack, nzTrack]);
  frame_b.R.w = frame_b.R.T*derPhiTrack;
\end{lstlisting}

Several design choices in the listing are worth noting.
\begin{itemize}
\item \textbf{Ground-plane lateral direction.} The lateral offset
      direction is the ground-plane normal
      $\{-\partial_s r_y,\partial_s r_x,0\}$ to the reference-line tangent, so the
      orientation of \texttt{frame\_b} is determined entirely by
      the spline first and second $s$-derivatives. No additional
      yaw spline is needed and there is no quaternion or
      Euler-angle integration in the joint.
\item \textbf{Horner evaluation.} The Monge cubic and its $s$- and
      $w$-partials are evaluated in Horner form. Horner's method is the
      backward-stable way to evaluate a polynomial, so it keeps the
      relative error small in the high-order derivatives that feed the
      surface triad and the acceleration-level constraint. It is not the
      cheapest form here, though. The Modelica tool expands integer powers
      and then shares common subexpressions across the local coordinate
      that all channels have in common, so an explicitly expanded
      polynomial compiles to faster code. Section~\ref{sec:benchmarks}
      measures both effects. We keep Horner form for the conditioning of
      the derivative chain, not for speed.
\item \textbf{D'Alembert constraint.} The two generalized forces
      in the $(s,w)$ directions are set to zero by projecting the
      resolved body-frame wrench $(\bm f^b,\bm\tau^b)$ onto the
      virtual displacements $\partial_s\bm r_b$ and $\partial_w\bm r_b$. This
      is equivalent to the principle of virtual work and closes
      the 2-DOF DAE in the form expected by a Modelica DAE solver
      without introducing auxiliary multiplier variables.
\item \textbf{$C^4$ continuity requirement.} The DAE solver
      differentiates the angular-velocity equation to form the
      acceleration-level constraint, which requires the second
      $s$-derivatives of every spline channel and the first mixed
      $sw$-derivatives. These exist because the export step of the
      companion fitting pipeline \cite{Verulkar:2026fit} enforces
      $C^4$ continuity at every interior knot and at the periodic
      wrap.
\end{itemize}

\section{Implementation and FMI export}\label{sec:implementation}

\subsection{Parameter record and spline storage}

All road geometry enters \texttt{TrackJoint} through a single
parameter record. Each of the six channels
$(r_x,r_y,r_z,z_1,z_2,z_3)$ is stored as a coefficient matrix of
size $N_{\mathrm{seg}}\times(d{+}1)$ with $d=5$, plus a single
break vector $\{s_i\}_{i=0}^{N_{\mathrm{seg}}}$ shared by all
channels. The break vector is uniform in $s$ in the present
implementation, which lets the break segment index $\mathrm{ix}$ lookup at a query $s$ be
done in constant time using a custom Modelica function. For nonuniform breaks,
the same lookup can be replaced by a binary search without any change to the rest of the joint
equations. Once $\mathrm{ix}$ is known, every spline evaluation is
\begin{equation}
y^{(q)}(s)
 = \sum_{j=q}^{d}
   \frac{j!}{(j-q)!}\,a_{\mathrm{ix},j}\,(s-s_{\mathrm{ix}})^{\,j-q},
\end{equation}
performed in Horner form through a custom Modelica function, with the requested derivative
order $q\in\{0,1,2,3\}$ passed as an argument. The function is
declared locally with the common segment index $\mathrm{ix}$, $s$, the polynomial degree and the
break vector partially applied, so each per-channel evaluation call only
selects its coefficient matrix and its derivative order.

\subsection{State selection and initialization}

The joint declares the four state variables
$(s,\dot s,w,\dot w)$ with \texttt{StateSelect.always} so that a
Modelica compiler chooses them as continuous
states regardless of the surrounding vehicle topology, eliminating
state-aliasing failures observed when the chassis floor was
otherwise free in $\bm e_z$. Initial values are exposed as
parameters and propagated to the \texttt{start} attribute of the 
corresponding states. A boolean override is provided to relax the 
state selection in cases where the user prefers the tool to choose.

\subsection{MultiBody integration}

{\sloppy
The joint exposes a single MultiBody \texttt{Frame\_b} connector
(from \texttt{Modelica.Mechanics.\allowbreak MultiBody.\allowbreak Interfaces}).
A \texttt{Connections.\allowbreak root(frame\_b.R)} statement
establishes the joint as the root of the orientation graph of the
kinematic subtree it heads, mirroring the pattern used by the
standard MultiBody joints. The body-frame wrench arriving on the
connector is rotated to world orientation by the inverse of the
triad,
\par}
\begin{lstlisting}
fFrameB = resolve1(frame_b.R, frame_b.f);
tFrameB = resolve1(frame_b.R, frame_b.t);
\end{lstlisting}
and used in the d'Alembert
relations~\eqref{eq:dAlembertS}--\eqref{eq:dAlembertW}.

In a complete vehicle assembly the joint never carries the chassis
on its own. The \texttt{frame\_b} of \texttt{TrackJoint} is
followed by a standard
\texttt{Modelica.\allowbreak Mechanics.\allowbreak MultiBody.\allowbreak Joints.\allowbreak Prismatic} heave joint
oriented along the surface \emph{binormal} $\bm n_z$, and by three
revolute joints that realize roll, pitch and yaw of the chassis
\emph{relative to the track triad} rather than relative to the
world frame. Note that $\bm n_z$ is the vertical, out-of-plane direction of
the patch. It is not the in-surface lateral normal $\bm n_y$, along which
the joint coordinate $w$ is already measured. The chassis therefore retains
its full six degrees of freedom, but each one is naturally aligned with the
road: the heave coordinate is the vertical ride height above the patch
along $\bm n_z$, and the three angles are the chassis attitude
as a driver would describe it. This decomposition is what enables
the initialization strategy discussed in Section~\ref{sec:results}.

\subsection{FMI 3.0 export}\label{sec:fmiexport}

In use, the joint is connected inside Dymola to a Modelica
MultiBody vehicle model and the closed system is exported as a
single FMI~3.0 model exchange FMU. The joint
contributes the four continuous states $(s,\dot s,w,\dot w)$ to
the exported state vector; the heave prismatic along the surface
binormal $\bm n_z$ (vertical, out-of-plane) and the three
track-relative roll/pitch/yaw revolute joints stacked on top of
\texttt{frame\_b} contribute a further four position states and
their derivatives, so that the chassis exposes its six degrees
of freedom relative to the road surface. The remaining states,
inputs and outputs of the FMU are those of the host vehicle
model. The
spline coefficient matrices of the joint appear as
\texttt{tunable} parameters of the combined FMU so that a single
FMU binary can be re-targeted to a new circuit at run-time
without re-compilation. This is the form consumed by the FMIOPT
direct-collocation transcription described in
Section~\ref{sec:results}.

\section{Results}\label{sec:results}

We illustrate the use of \texttt{TrackJoint} in two settings.

\subsection{Time-domain simulation}

A Modelica experiment is constructed as shown in Figure~\ref{fig:dymola} 
in order to validate the track surface produced by \texttt{TrackJoint}.
In this experiment, a point-mass is driven around the Singapore Grand Prix Circuit 
along the minimum curvature line from the companion pipeline \cite{Verulkar:2026fit} 
with the joint as the only constraint. The resultant road surface is compared against 
the meshed terrain from the LiDAR point cloud of the Singapore Grand Prix Circuit using rFPro Terrain Server.

\begin{figure}[t]
  \centering
  \includegraphics[width=0.48\textwidth]{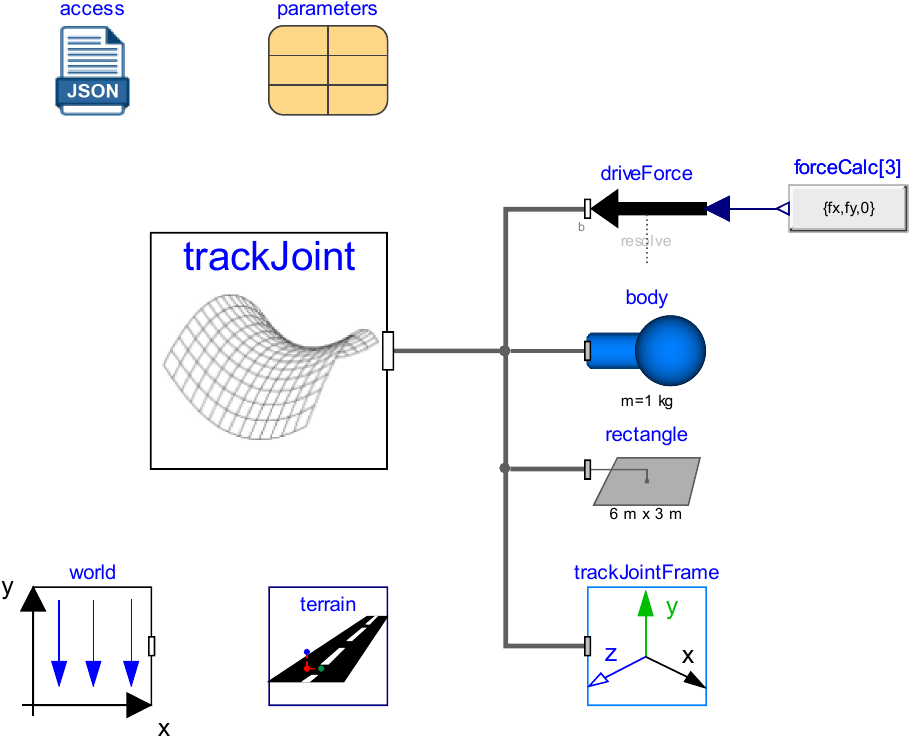}
  \caption{Modelica validation experiment of a point mass 
  driven around the track using \texttt{TrackJoint}.}
  \label{fig:dymola}
\end{figure}

\begin{figure}[t]
  \centering
  \includegraphics[width=0.48\textwidth]{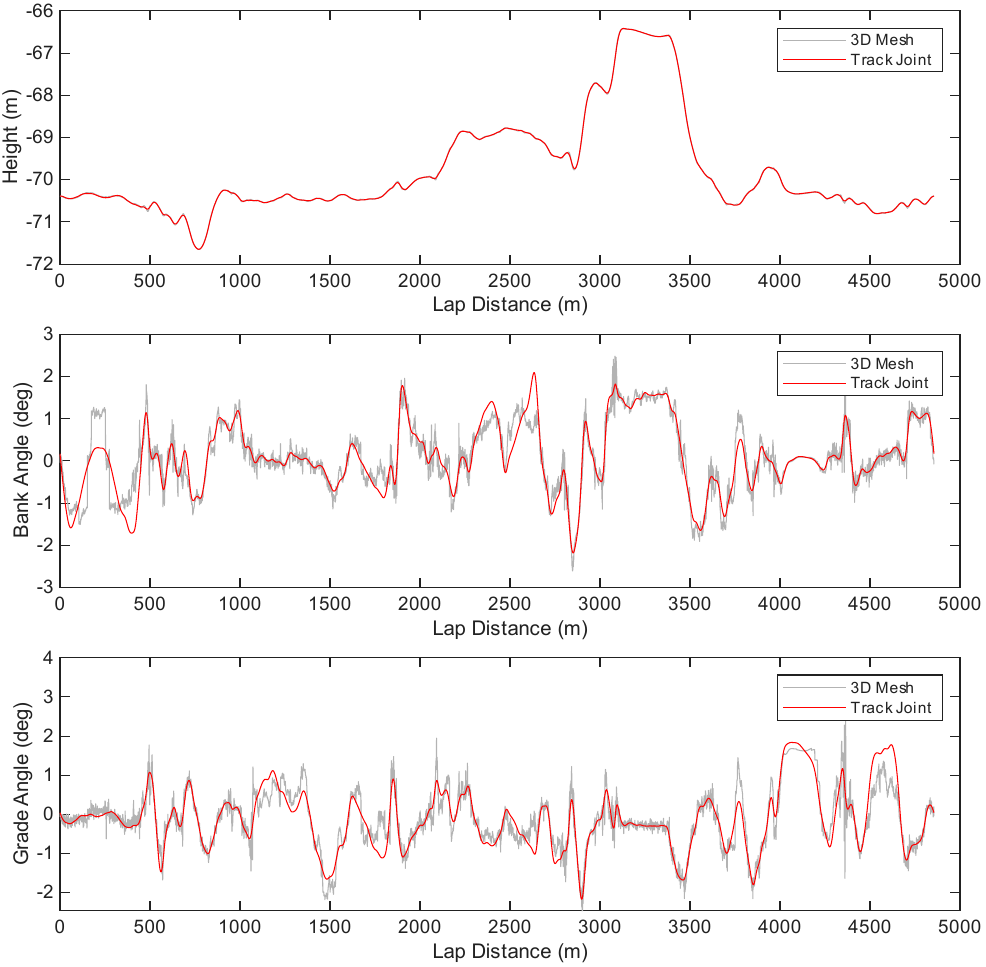}
  \caption{Validation of the \texttt{TrackJoint} surface output
  against the measured track geometry of the Singapore Grand Prix Circuit from LiDAR point cloud.}
  \label{fig:validation}
\end{figure}

\begin{figure}[t]
  \centering
  \includegraphics[width=0.48\textwidth]{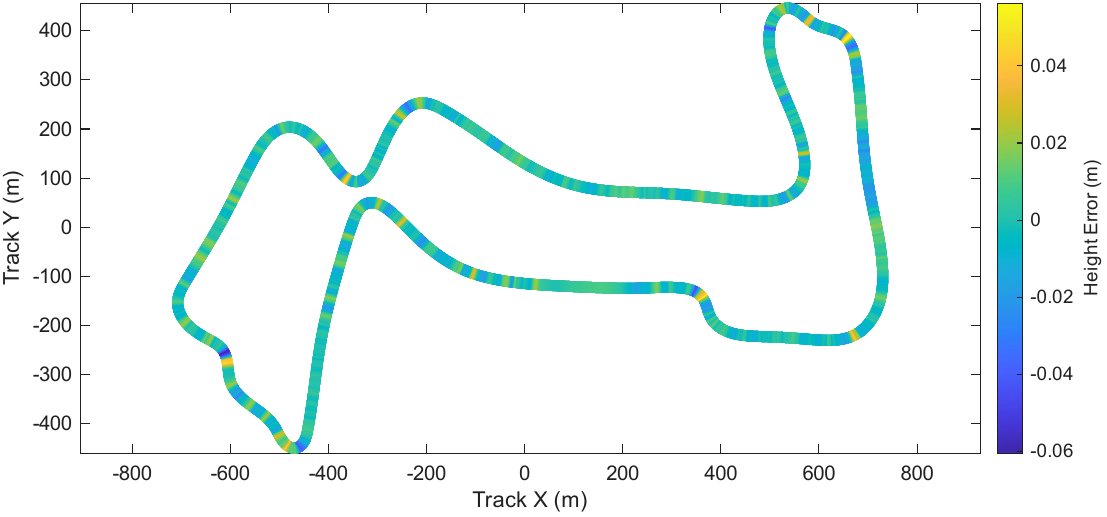}
  \caption{Path of the point mass driven through the \texttt{TrackJoint} and 
  the error in absolute height against the 3D mesh track surface of the 
  Singapore Grand Prix Circuit constructed from LiDAR point cloud.}
  \label{fig:singapore}
\end{figure}

Figures~\ref{fig:validation} and~\ref{fig:singapore} show the simulation results from Dymola 2026x Refresh 1,
which compare the reconstructed track surface geometry produced by the joint with the 3D mesh road geometry
from the LiDAR point cloud. The reconstructed surface at $\bm r_b(s,w)$ closely tracks the 3D mesh road geometry
in absolute height, grade and bank angles across the lap, which confirms that the joint faithfully reproduces
the measured road surface in time-domain simulation.

\subsection{Minimum lap-time direct collocation}

\begin{figure}[t]
  \centering
  \includegraphics[width=0.48\textwidth]{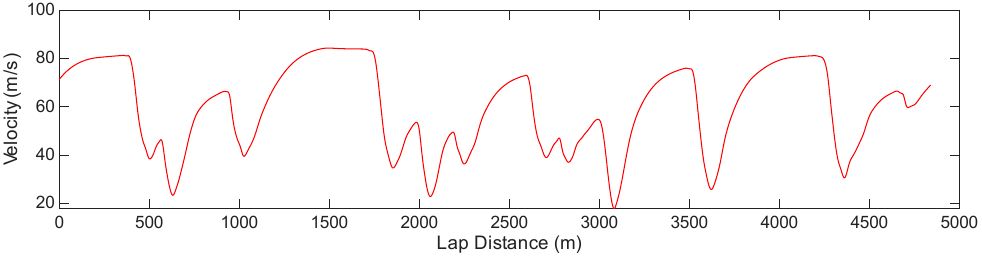}
  \caption{Optimized velocity profile of a multibody vehicle model driven around the Singapore Grand Prix Circuit using \texttt{TrackJoint} in a FMIOPT minimum lap-time problem.}
  \label{fig:velocity}
\end{figure}

\begin{figure}[t]
  \centering
  \includegraphics[width=0.48\textwidth]{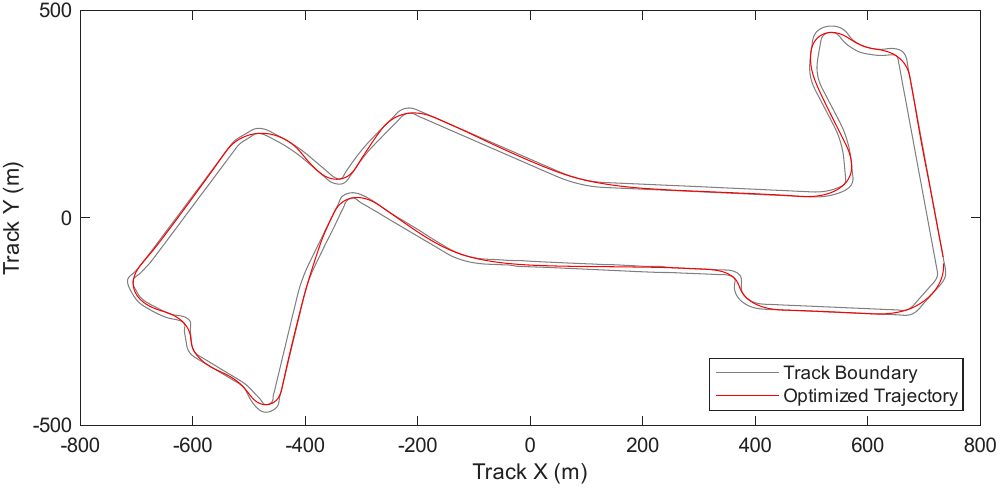}
  \caption{Optimized racing line of a multibody vehicle model driven around the Singapore Grand Prix Circuit using \texttt{TrackJoint} in a FMIOPT minimum lap-time problem.}
  \label{fig:racingline}
\end{figure}

The multibody vehicle-plus-track-joint FMU described in
Section~\ref{sec:fmiexport} is then handed to the FMIOPT
toolchain, which wraps it inside a CasADi
direct-collocation NLP~\cite{Andersson:2019} solved by
IPOPT~\cite{WaechterBiegler:2006} with the MUMPS linear solver.
Because the joint exposes the parametric Jacobian factor through
$\partial_s\bm r_b$, the standard time-to-space
transformation~\eqref{eq:dtds}
is built directly from the splines without any further model
modification.
On every circuit on which the splines are anchored
on the minimum-curvature line of \cite{Verulkar:2026fit}, the
denominator of \eqref{eq:dtds} stays bounded away from zero on the
full track strip and the NLP converges to the IPOPT
tolerances; on circuits on which the splines are anchored on the
geometric center-line, the same NLP exhibits the convergence
failure documented in the companion paper. This is what
Proposition~\ref{prop:rank} means in practice. The denominator of
\eqref{eq:dtds} is the same factor $1-\kappa w$ that sets the rank of the
constraint Jacobian \eqref{eq:Amat}, so a reference line that breaks
\eqref{eq:injectivity} hurts the transcription and the joint at the same
time. Section~\ref{sec:benchmarks} gives the margin for all 23 circuits
and the solve statistics for Singapore.

A second, equally important role of \texttt{TrackJoint} in the
direct-collocation pipeline is the \emph{initialization} of the
NLP. Because the chassis degrees of freedom are expressed
relative to the track triad through the heave prismatic and the
three track-relative revolute joints stacked on top of
\texttt{frame\_b}, a feasible initial guess for an entire lap is
produced by setting every chassis coordinate (heave, roll, pitch,
yaw, lateral offset $w$) and every velocity except $\dot s$ to
zero, and prescribing $s(\tau)$ as a linear ramp from $0$ to $L$
over the normalized time horizon $\tau\in[0,1]$. This guess is
trivial to generate, requires no preliminary simulation and no
knowledge of the circuit beyond its length, yet it places the
vehicle exactly on the road surface with its body axes aligned
to the track triad at every collocation point. For Singapore the guess is
$324$ copies of the same zero vector for every chassis coordinate together
with a linear ramp $s(\tau)=\tau L$, and from that IPOPT reaches
\texttt{Optimal Solution Found} in $264$ iterations
(Table~\ref{tab:nlp}). It works because the track-relative decomposition
makes the guess satisfy the road constraint by construction. With all
chassis coordinates at zero the vehicle sits on $\bm r_b(s,0)$ with
$\mathbf R_b$ equal to the surface triad, so the two d'Alembert
relations already hold at the first iterate and only the vehicle
dynamics are left to resolve. In our experience
this is what allows IPOPT to take its first steps without
running into the orientation-singular regions that plague
world-frame chassis parametrizations, and it is the practical
reason why the same NLP converges on every circuit fitted by the
companion pipeline.

Figures~\ref{fig:velocity} and~\ref{fig:racingline} show the result
of a minimum lap-time optimization problem of a full-detail
multibody vehicle model driven around the Singapore Grand Prix
Circuit through \texttt{TrackJoint}. Figure~\ref{fig:racingline}
plots the optimized racing line on the track strip generated by the
joint splines, and Figure~\ref{fig:velocity} the corresponding
optimized velocity profile along the lap. Both quantities are
produced in a single converged solve of the FMIOPT
direct-collocation NLP, illustrating that the \texttt{TrackJoint}
component yields a well-conditioned NLP suitable for full-circuit
minimum lap-time optimization of a detailed multibody vehicle.

\subsection{Quantitative benchmarks}\label{sec:benchmarks}

All measurements below were taken on one desktop workstation with
Dymola 2026x Refresh~1. Timings are medians of five repetitions.

\paragraph{Circuit coverage.}
The joint has been run on the 23 Formula~1 circuits of the
2026 calendar. Every circuit is fitted with a uniform $5$\,m knot
spacing, which gives $657$--$1387$ quintic segments for track lengths of
$3283$--$6934$\,m. A complete spline parameter record is about
$1.2$\,MB of JSON. Table~\ref{tab:circuits} lists the extremes. The
injectivity margin \eqref{eq:injectivity} is above one on \emph{all} 23
circuits. The tightest is $1.03$ at Monaco, where the reference line comes
down to a $9.3$\,m radius. The set covers both very tight corners (Monaco,
Singapore) and steep camber (Zandvoort, up to $16.2^\circ$ of bank).

\begin{table}[t]
\centering
\caption{Geometry of the fitted circuits: track length $L$, number
of quintic segments, minimum reference-line radius, maximum bank
angle and injectivity margin $\min_s 1/(|\kappa| b)$. Extremes of the
23-circuit set.}
\label{tab:circuits}
\footnotesize
\begin{tabular}{lrrrrr}
\hline
Circuit & $L$\,[m] & $N_\mathrm{seg}$ & $R_\mathrm{min}$\,[m]
        & bank\,[$^\circ$] & margin \\
\hline
Monaco       & 3283 &  657 &  9.3 &  5.2 & 1.03 \\
Zandvoort    & 4185 &  837 & 31.2 & 16.2 & 2.32 \\
Singapore    & 4858 &  972 & 23.0 &  2.1 & 1.56 \\
Spa          & 6934 & 1387 & 22.5 &  4.6 & 1.42 \\
Silverstone  & 5798 & 1160 & 41.2 &  1.9 & 2.38 \\
\hline
all 23       & 3283 &  657 &  9.3 &  1.5 & \textbf{1.03} \\
(range)      & 6934 & 1387 & 41.2 & 16.2 & 2.89 \\
\hline
\end{tabular}
\end{table}

\paragraph{Per-step evaluation cost.}
To measure the cost of the joint on its own, a unit point mass was
driven one full lap through it with a fixed-step fourth-order
Runge--Kutta integrator over $180{,}000$ steps, so that every variant makes
exactly $720{,}000$ model evaluations. The released Horner version costs
$3.54$\,s of integration CPU time, which is $4.9\,\mu$s per model
evaluation for the whole 2-DOF joint including the triad and both
rotation-rate vectors. Writing every polynomial out in expanded power form,
and changing nothing else, gives bit-identical trajectories for $2.63$\,s,
or $26\%$ less. The expanded form wins because Dymola turns integer powers
into repeated multiplications and then shares $x^2,\dots,x^5$ across all
eight channels and all four derivative orders. The Horner recursion instead
gives each of the $26$ evaluations its own chain of dependent
multiply--adds. We keep Horner form for the accuracy reason below, not for
speed.

\paragraph{Evaluation accuracy.}
We evaluated both forms in double precision at every segment of the
Singapore fit and compared them against a 60-digit reference evaluation of
the same polynomial. Over $972$ segments, all six channels and derivative
orders $0$--$3$, the largest relative error is $2.6\times10^{-12}$ for
Horner and $5.0\times10^{-12}$ for the expanded form. So Horner is at best
a factor $1.9$ more accurate, and on some channels the order reverses.
Both sit far below any tolerance that matters physically at this knot
spacing. The difference only starts to count for much coarser knots.

\paragraph{FMU size and NLP performance.}
The vehicle-plus-joint model exports to an FMI~3.0 model-exchange
FMU with a $35.5$\,MB shared library ($10.9$\,MB archived). On the
Singapore circuit the transcription uses $323$ intervals
($324$ collocation points) and the solve is warm-started only by
the trivial guess described above. Table~\ref{tab:nlp} reports the
back-to-back result for the two FMUs, which differ \emph{only} in
the body of the spline evaluation function; both consumed
byte-identical problem data. Both reach
\texttt{Optimal Solution Found}, and the lap times agree to
$3.3$\,ms out of $90.6$\,s, a relative difference of
$4\times10^{-5}$. The iteration counts differ by $6\%$. We read that as the
sensitivity of a nonconvex interior-point path to perturbations of order
$10^{-12}$, not as a real advantage either way. On this evidence the choice
of evaluation form makes no measurable difference to convergence.

\begin{table}[t]
\centering
\caption{Singapore minimum lap-time NLP with the two FMU builds,
which differ only in the spline evaluation function. Identical track,
setup, guess, mesh and solver options.}
\label{tab:nlp}
\footnotesize
\begin{tabular}{lrr}
\hline
 & Horner & expanded \\
\hline
IPOPT status            & optimal & optimal \\
Iterations              & 264 & 249 \\
IPOPT time [s]          & 464.4 & 442.5 \\
Final NLP error         & $8.0\times10^{-4}$ & $3.8\times10^{-4}$ \\
Lap time [s]            & 90.6462 & 90.6429 \\
\hline
\end{tabular}
\end{table}

The FMU, not the linear algebra, dominates the cost of a solve. Of the
$464$\,s total, $167$\,s go into $266$ objective-gradient evaluations
($627$\,ms each), $120$\,s into $389$ constraint evaluations ($310$\,ms
each) and $23$\,s into $266$ constraint-Jacobian evaluations ($85$\,ms
each). Making the road model cheaper to evaluate therefore pays off
directly, which is why the $26\%$ figure above is worth having.

\paragraph{Comparison with the Claytex \texttt{Rail} joint.}
Table~\ref{tab:rail} sets the present component beside the
1-DOF rail joint \cite{Briant:2021}. The differences are
about what the two components can represent.
\texttt{Rail} holds a body on a curve and has to be told its orientation,
either by tying it to the tangent or through three extra angle tables.
\texttt{TrackJoint} holds a body on a surface and works the whole
orientation triad out from the surface partials. Timing the two against
each other would not mean much, since they solve different problems. The
point is that the second degree of freedom, the analytic triad and the
second d'Alembert relation all come from six scalar splines and no extra
tabulated data.

\begin{table}[t]
\centering
\caption{Claytex \texttt{Rail} versus \texttt{TrackJoint}.}
\label{tab:rail}
\footnotesize
\begin{tabular}{lll}
\hline
 & \texttt{Rail} & \texttt{TrackJoint} \\
\hline
Free DOF      & 1 ($s$)        & 2 ($s,w$) \\
Geometry      & space curve    & Monge-patch surface \\
Orientation   & rail tangent or & analytic triad from \\
              & 3 angle tables  & surface partials \\
Input data    & $xyz$ or $xyz\phi$ & 6 quintic splines \\
              & sample table    & $+$ shared knots \\
Constraint    & 1 projection   & 2 d'Alembert relations \\
Camber/crown  & Constant (centerline) & $z(s,w)$ polynomial \\
\hline
\end{tabular}
\end{table}

\section{Limitations}\label{sec:limitations}

There are limits to what the component does today, and to what the
measurements above can show.

\emph{Dependence on the fitting pipeline.} The joint uses splines it does
not produce, and it cannot tell at run time that a reference line was
badly anchored. If \eqref{eq:injectivity} is broken, the constraint
Jacobian quietly loses rank in the offending corner. A cheap guard would be
to export $\min_s 1/(|\kappa| b)$ with the splines and assert on it at
initialization. The present version does not do this.

\emph{Validation scope.} We validated the surface against LiDAR terrain for
one circuit and one point-mass run, and the optimization results use one
vehicle, one setup and one initial guess per circuit. We have not run a
spline-resolution sensitivity study or a multi-seed convergence study, so
our comments about robustness are experience rather than measurement.

\emph{Availability.} \texttt{TrackJoint} is proprietary. We give the
formulation in full here so that it can be reimplemented, but the
component, the vehicle models and the fitted circuit data are not public.

\emph{Implementation restrictions.} The break vector is uniform, so
non-uniform knots would need a search in place of the constant-time index
lookup. The joint assumes a closed circuit and reduces $s$ modulo the lap
length. It exposes one \texttt{frame\_b}, so four-wheel contact-patch
queries need four instances sharing one parameter record.

\section{Conclusions}

Packaging a racetrack surface as an ordinary MultiBody joint works well in
Modelica. \texttt{TrackJoint} exposes only a \texttt{frame\_b} connector
and two joint coordinates, so it plugs into the standard library like any
elementary joint, and the whole road model comes down to six scalar splines
and a shared knot vector.

Three things came out of this work. First, the regularity of the patch is
not a modelling detail. It is the condition under which the joint can be
solved at all: Proposition~\ref{prop:rank} shows the constraint Jacobian
loses rank exactly on the Frenet fold, so $\min_s 1/(|\kappa| b)>1$ on the
input splines is what makes the DAE well posed. All 23 circuits fitted here
meet it, the tightest being $1.03$ at Monaco. Second, writing the chassis
coordinates relative to the track triad makes the minimum lap-time problem
easy to start. Setting every chassis coordinate to zero and ramping $s$
puts the vehicle on the road with its axes aligned at every collocation
point, with no preliminary simulation. Third, about two thirds of a
$464$\,s Singapore solve happens inside the FMU. The $26\%$ gap we measured
between two mathematically identical polynomial forms is a reminder that
what counts is the generated code, not the operation count on paper.

The next steps come straight out of Section~\ref{sec:limitations}:
assert the injectivity margin at initialization, find out what the cubic
lateral term buys on a circuit that needs it, swap the uniform-knot storage
for an adaptive one, and build a contact-patch variant in which four
instances share one parameter record.

\section*{Acknowledgements}

The component builds on the geometric formulation of
\textcite{LimebeerVanDeMerwe:2025}, with input splines from the
companion fitting work \cite{Verulkar:2026fit}. We gratefully
acknowledge the Modelica Standard Library MultiBody package,
CasADi~\cite{Andersson:2019} and IPOPT~\cite{WaechterBiegler:2006}.

\printbibliography

\end{document}